\documentclass[runningheads,a4paper]{llncs}

\usepackage{amssymb}
\usepackage{graphicx}
\usepackage{tikz}

\usepackage{algorithm2e}

\usepackage{url}
\urldef{\mailsa}\path|{joseluis.balcazar,cristina.tirnauca}@unican.es|
\newcommand{\keywords}[1]{\par\addvspace\baselineskip
\noindent\keywordname\enspace\ignorespaces#1}

\def\LL{\hbox{$\cal L$}}
\def\st{\bigm|}
\def\uc{\hbox{\rm uc}}
\def\lc{\hbox{\rm lc}}
\def\LC{\hbox{\rm LC}}
\def\0{\emptyset}

\begin{document}

\mainmatter  % start of an individual contribution

\title{Border Algorithms for Computing \\ Hasse Diagrams of Arbitrary Lattices\thanks{This work 
has been partially supported by project FORMALISM (TIN2007-66523)
 of Programa Nacional de Investigaci\'on, Ministerio de Ciencia
 e Innovaci\'on (MICINN), Spain, by the Juan de la Cierva contract JCI-2009-04626 of the same ministry, and by the Pascal-2 Network of the European Union.}}

\titlerunning{Border Algorithms for Computing Hasse Diagrams of Arbitrary Lattices}


\author{Jos\'e L Balc\'azar
\and Cristina T\^\i{}rn\u{a}uc\u{a}}
\authorrunning{Jos\'e L Balc\'azar, Cristina T\^\i{}rn\u{a}uc\u{a}}

\institute{Departamento de Matem\'aticas, Estad\'\i{}stica y Computaci\'on\\
Universidad de Cantabria\\
Santander, Spain\\
\mailsa\\
%\url{http://www.springer.com/lncs}
}

\toctitle{Border Algorithms for Computing Hasse Diagrams of Arbitrary Lattices}
\tocauthor{Jos\'e L Balc\'azar and Cristina T\^\i{}rn\u{a}uc\u{a}}
\maketitle

\begin{abstract}
The Border algorithm and the iPred algorithm find the
Hasse diagrams of FCA lattices. We show that they can
be generalized to arbitrary lattices. In the case of
iPred, this requires the identification of a join-semilattice
homomorphism into a distributive lattice.
\keywords{Lattices, Hasse diagrams, border algorithms}
\end{abstract}

\section{Introduction}
Lattices are mathematical structures with many applications in
computer science; among these, we are interested in fields like 
data mining, machine learning, or knowledge discovery in databases. 
One classical use of lattice theory is in formal concept analysis 
(FCA) \cite{GanWil99}, where the concept lattice with its diagram 
graph allows for the visualization and summarization of data in a
more concise representation. In the Data Mining community, the same
mathematical notions (often under additional ``frequency''
constraints that bound from below the size of the support set) 
are studied under the banner of Closed-Set Mining
(see~e.g.~\cite{ZakHsi05}).

In these applications, 
data consists of \emph{transactions}, also called \emph{objects},
each of which, besides having received a unique identifier, consists
of a set of \emph{items} or \emph{attributes} taken from a previously
agreed finite set. 
A concept is a pair formed by a set of transactions ---the \emph{extent} 
set or \emph{support set} of the concept--- and a set of attributes 
---the \emph{intent} set of the concept---  defined as the set of all
those attributes that are shared by  all the transactions present in 
the extent. Some data analysis processes are based on the family of
all intents (the ``closures'' stemming from the dataset); but others
require to determine also their order relation, which is % always
a finite lattice, in the form of a line graph (the \emph{Hasse diagram}). 

Existing algorithms can be
divided into three main types: the ones that only generate the set of
concepts, the ones that first generate the set of concepts and then
construct the Hasse diagram, and the ones that construct the diagram
while computing the lattice elements (see \cite{ZakHsi05}, and 
also \cite{GoMiAl95,KuzObi01} and
the references therein). The goal is to obtain the concept lattice in linear time in the number of concepts because this number is, most of the times, already exponential in the number of attributes, making the task of getting polynomial algorithms in the size of the input rather impossible.

%Indeed,
One widespread use of concepts or closures is the generation of
implications or of partial implications (also called association rules).
Several data mining algorithms aim at processing large
datasets in time linear in the size of the closure space,
and explore closed sets individually; these solutions tend
to drown the user under a deluge of partial implications. More 
sophisticated works attempt at providing selected ``bases'' of
partial implications; the early proposal in \cite{Luxe91} requires
to compute immediate predecessors, that is, the Hasse diagram. 
Alternative proposals such as the
Essential Rules of \cite{AggYu01b} or the equivalent Representative 
Rules of \cite{Krys98} (of which a detailed discussion with new
characterizations and an alternative basis proposal appears in
\cite{Balc10c}) require to process predecessors of closed sets 
obeying tightly certain support inequalities; these algorithms
also benefit from the Hasse diagram, as the slow alternatives are
blind repeated traversal of the closed sets in time quadratic 
in the size of the closure space,
or storage of all predecessors of each closed set, which soon
becomes large enough to impose a considerable penalty on the
running times.

% rephrased to save one line and get better page breaks

The problem of constructing the Hasse diagram of an
arbitrary finite lattice is less studied. One algorithm that has a better  
worst case complexity than various previous works is described in \cite{NouRay99}.  
From our ``arbitrary lattices'' perspective,
its main drawback is that it requires the availability of a
\emph{basis} from which each element of the lattice can be derived. In
the absence of such a subset, one may still use this % e above mentioned
algorithm (at a greater computational cost) to output the
Dedekind-MacNeille completion \cite{DavPri91} of the given lattice,
which in our case is isomorphic to the lattice itself. The algorithm
is also easily adaptable to concept lattices, where indeed a basis is
available immediately from the dataset transactions.

We consider of interest to have available further, faster algorithms
for arbitrary finite lattices; we have two reasons for this aim.
First, many (although not all) algorithms constructing Hasse
diagrams traverse concepts in layers defined by the size of
the intents; our explorations about association rules sometimes
require to follow different orderings, so that a more abstract
approach is helpful; second, we keep in mind the application area
corresponding to certain variants of implications and database
dependencies that are characterized by lattices of equivalence
relations, so that we are interested in laying a strong
foundation that gives us a clear picture of the applicability
requirements for each algorithm constructing Hasse diagrams
in lattices other than powerset sublattices.

Of course, we expect that FCA-oriented
algorithms could be a good source of inspiration for the design
of algorithms applicable in the general case. An example that such 
an extension can be done is the algorithm in \cite{VaMiLe00} 
(see Section \ref{s:Border} for more details), whose highest-level 
description matches the general case of arbitrary lattices; 
nevertheless, the actual implementation described in \cite{VaMiLe00} 
works strictly for formal concept lattices, so that further
implementations and complexity analyses are not readily available 
for arbitrary finite lattices.

The contribution of the present paper supports the same idea: we show how
two existing algorithms that build the Hasse diagrams of a concept
lattice can be adapted to work for arbitrary lattices. Both algorithms
have in common the notion of \emph{border}, which we (re-)define and
formalize in Section \ref{s:Border}, after presenting some preliminary
notions about lattice theory in Section \ref{s:Prel}; our approach
has the specific interest that the notion of border is given just
in terms of the ordering relation, and not in terms of a set of
elements already processed as in previous references 
(\cite{BSVG09,MarEkl08,VaMiLe00}); yet, the notions are equivalent.
We state and prove properties of borders and describe the
\emph{Generalized Border Algorithm};
whereas the algorithm reads, in high level, exactly
as in previous references, its validation is new, as
previous ones depended on the lattice being an FCA lattice.
In Section 
\ref{s:iPA} 
we introduce the
\emph{Generalized iPred Algorithm}, exporting the iPred
algorithm of FCA lattices \cite{BSVG09} to arbitrary lattices, 
after arguing its correctness.
%  (Section \ref{s:iPA}). 
This task
is far from trivial and is our major contribution, since the
existing 
 rendering and validation 
 of the iPred algorithm 
%version
relies 
again 
extensively on the fact that it is being applied to an FCA lattice,
and even performs operations on difference sets that may not
belong to the closure space.
Concluding remarks and future work ideas are presented in Section \ref{s:Conc}.

\section{Preliminaries}\label{s:Prel}

We develop all our work in terms of lattices and
semilattices; see \cite{DavPri91} as main source. 
All our structures are \emph{finite}.
A \emph{lattice} is a partially ordered set in which
every nonempty subset has a meet (greatest lower bound) and
a join (lowest upper bound). If only one of these 
two operations is guaranteed to be available a priori, 
we speak of a \emph{join-semilattice} or a 
\emph{meet-semilattice} as convenient. 
Top and bottom elements
are denoted $\top$ and $\bot$, respectively.
Lower case letters, possibly with primes, and taken
usually from the end of the latin alphabet denote
lattice elements: $x$, $y'$. Note that Galois 
connections are not explicitly present in this paper, 
so that the ``prime'' notation does not refer to the
operations of Galois connections.

Finite semilattices can be extended into lattices
by addition of at most one further element \cite{DavPri91}; for instance,
if $(\LL,\leq,\lor)$ is a join-semilattice with bottom 
element $\bot$, one can define a meet operation as follows:
$\bigwedge X = \bigvee \{y \st \hbox{$\forall x \in X,$} y \leq x\}$;
the element $\bot$ ensures that this set is nonempty.
Thus, if the join-semilattice lacks a bottom element,
it suffices to add an ``artificial'' one to obtain a lattice.
A dual process is obviously possible in meet-semilattices.

Given two join-semilattices $(S,\lor)$ and $(T, \lor)$, 
a \emph{homomorphism} % of join-semilattices 
is a 
function $f: S \rightarrow T$ such that
$f(x \lor y) = f(x) \lor f(y)$. Hence $f$ is just 
a homomorphism of the two semigroups associated 
with the two semilattices. If $S$ and $T$ both include 
a bottom element $\bot$, then $f$ should also be a 
monoid homomorphism, i.e. we additionally require 
that $f(\bot) = \bot$. % Meet-semilattices and 
Homomorphisms of meet-semilattices 
and of lattices are defined similarly.
It is easy to check that $x\leq y \Rightarrow f(x)\leq f(y)$
for any homomorphism $f$;
the converse implication, thus the
equivalence
$x\leq y \Leftrightarrow f(x)\leq f(y)$, 
is also true for injective~$f$ but not guaranteed in general.

We must point out here a simple but crucial fact that 
plays a role in our later developments: given a homomorphism $f$
between two join-semilattices $S$ and $T$, if we extend both
into lattices as just indicated, then $f$ is \emph{not} necessarily 
a lattice homomorphism; for instance, there could be elements of $T$
that do not belong to the image set of $f$, and they may become meets
of subsets of $T$ in a way that prevents them to be the image of the 
corresponding meet of $S$. For one specific example, see
Figure~\ref{fg:lattices}: 
consider the two join-semilattices defined by the solid lines,
where the numbering defines an injective homomorphism from the 
join-semilattice
in (a) to the join-semilattice in (b). Both lack a bottom element. Upon
adding it, as indicated by the broken lines, in lattice (a) the
meets of 1 and 2 and of 1 and~3 coincide, but the meets of their
corresponding images in (b) do not; for this reason, the 
homomorphism cannot be extended to the whole lattices.

\begin{figure}
\begin{center}
\begin{tikzpicture}
% Left pane:
\draw (2,0) -- (0,2) -- (-2,0)   (0,2) -- (0,0); 
\draw [dashed] (2,0) -- (0,-2) -- (-2,0)   (0,-2) -- (0,0); 
\draw (0,2) node [fill=white] {$\top$}; 
\draw (-2,0) node [fill=white] {1};
\draw (0,0) node [fill=white] {2};
\draw (2,0) node [fill=white] {3}; 
\draw (0,-2) node [fill=white] {$\bot$};
\draw (0,-3) node {(a)};
% Right pane - must learn to do it by adding an offset instead of hardwired
\draw (6,2) -- (4,0); 
\draw (6,2) -- (6,0); 
\draw (6,2) node [fill=white] {$\top$}  -- (8,0); 
\draw (4,0) node [fill=white] {1} -- (5,-1);
\draw (6,0) node [fill=white] {2} -- (5,-1);
\draw [style=dashed] (5,-1) node [fill=white] {4} -- (6,-2);
\draw [style=dashed] (8,0) node [fill=white] {3} -- (6,-2) node [fill=white] {$\bot$};
\draw (6,-3) node {(b)};
\end{tikzpicture}
\caption{Two join-semilattices converted into lattices}
\label{fg:lattices}
\end{center}
\end{figure}
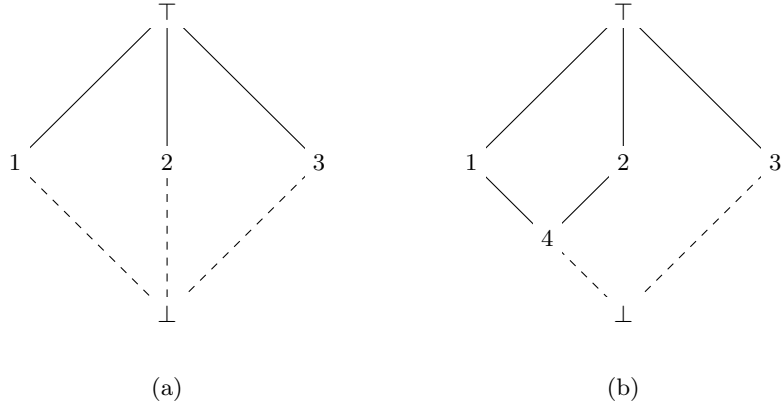

However, the following does hold:

\begin{lemma}
\label{l:joinmorph}
Consider two join-semilattices $S$ and $T$,
and let $f :S \to T$ be a %$(\leq,\lor)$-
homomorphism.
%between them
 After extending both semilattices
into lattices, 
$f(\bigwedge Y) \leq \bigwedge f(Y)$ for all $Y\subseteq S$.
\end{lemma}

This is immediate to see by considering that
$\bigwedge Y \leq y$ for all $y\in Y$, hence
$f(\bigwedge Y) \leq f(y)$ for all such $y$, 
and the claimed inequality follows.

We employ $x<y$ as the usual shorthand: 
$x\leq y$ and $x\neq y$.
We denote as $x\prec y$ the fact that $x$ is an
immediate predecessor of $y$ in $\LL$, that is,
$x < y$ and, for all $z$, 
$x < z \leq y$ implies $z=y$
(equivalently, $x\leq z < y$ implies $x=z$).

We focus on algorithms that have access to an
underlying finite lattice $\LL$ of size~$|\LL|=n$,
with ordering denoted $\leq$; abusing language 
slightly, we denote by $\LL$ as well its carrier set.
The \emph{width} $w(\LL)$ of the lattice $\LL$ 
is the maximum size of an antichain
(a subset of $\LL$ formed by pairwise incomparable elements).
The lattice is assumed to be available for our
algorithms in the form of an abstract data type 
offering an iterator that traverses all the elements of the 
carrier set, together with the operations of testing 
for the ordering 
(given $x$, $y\in\LL$, find out whether $x\leq y$) 
and computing the meet $x\land y$ and join $x\lor y$ 
of $x$, $y\in\LL$; also the constants $\top\in\LL$ and $\bot\in\LL$
are assumed available. 

% We assume as well that computing any of all these
% operations takes constant time. 

The algorithms we consider are to perform the
task of constructing explicitly the Hasse diagram
(also known as the reflexive and transitive reduction)
of the given lattice: $H(\LL)=\{(x,y)\st x\prec y\}$.
By projecting the Hasse diagram along the first or
the second component we find our crucial ingredients:
the well-known upper and lower covers.

\begin{definition}
\label{def:covers}
The upper cover of $x\in\LL$ is 
$\uc(x) = \{ y \st x\prec y \}$. 
The lower cover of $y\in\LL$ is 
$\lc(y) = \{ x \st x\prec y \}$.
\end{definition}

The following immediate fact is stated separately
just for purposes of easy later reference:

\begin{proposition}
\label{p:covers}
If $x<y$ then there is $z\in\uc(x)$
such that $x\prec z\leq y$; and there is $z'\in\lc(y)$
such that $x\leq z'\prec y$.
\end{proposition}

We will use as well yet another easy technicality:

\begin{lemma}
\label{l:twoprecs}
If $x_1 \prec y$ and $x_2 \prec y$, with $x_1\neq x_2$
then $x_1 \lor x_2 = y$.
\end{lemma}

\begin{proof}
Since $y \geq x_1$ and $y \geq x_2$ we have $y \geq x_1 \lor x_2$. 
Then, $x_1\neq x_2$ implies that they are mutually incomparable,
since otherwise the smallest is not an immediate predecessor of $y$;
this implies that $y \geq x_1 \lor x_2 > x_1$, whence $y = x_1 \lor x_2$
as $x_1\prec y$.\qed
\end{proof}

\section{The Border Algorithm in Lattices}\label{s:Border}

The algorithms we are considering here have in common
the fact that they traverse the lattice and explicitly 
maintain a subset of the elements seen so far: those 
that still might be used to identify new Hasse edges.
This subset is known as the ``border'' and, as it
evolves during the traversal, actually each element 
$x\in\LL$ ``gets its own border'' associated as the 
algorithm reaches it. The border associated to an
element may be potentially used to construct new
edges touching it (although these edges may not 
touch the border elements themselves): more precisely,
operations on the border for $x$ will result in $\uc(x)$, hence in the Hasse edges of the form $(x,z)$.

In previous references the border is defined in terms
of the elements already processed, and its properties
are mixed with those of the algorithm that uses it.
Instead,
we study axiomatically the properties of the notion of
``border'' on itself, always as a function of the element
for which the border will be considered as a source
of Hasse edges, in a manner that is independent of
the fact that one is traversing the lattice. 
This allows us to clarify which abstract
properties are necessary for border-based algorithms,
so that we can generalize them to arbitrary lattices,
traversed in flexible ways. 
Our key definition is, therefore:

\begin{definition}
\label{def:border}
Given $x\in\LL$ and $B\subseteq\LL$, $B$ is a 
\emph{border for} $x$ if the following properties hold:
\begin{enumerate}
\item
$\forall y\in B \, (y \not\leq x)$;
\item
$\forall z \, (x\prec z \Rightarrow \exists y \in B \, (y\leq z))$.
\end{enumerate}
\end{definition}

That is, $x$ is never above an element of a border,
but each upper cover of $x$ is; this last condition
is equivalent to: all elements strictly above $x$ 
are greater than or equal to some element of the border. 
Since $x\leq (x\lor y)$ always
holds and $x = (x\lor y)$ if and only if $y\leq x$, we get:

\begin{lemma}
\label{l:easy}
Let $B$ be a border for $x$. Then
$\forall y \in B \, (x < x\lor y)$.
\end{lemma}

All our borders will fulfill an extra ``antichain''
condition; the only use to be made of this fact is
to bound the size of every border by the width of
the lattice.

\begin{definition}
\label{def:properborder}
A border $B$ is \emph{proper} if every two different 
elements of $B$ are mutually incomparable.
\end{definition}

The key property of borders, that shows how to 
extract Hasse edges from them, is the following:

\begin{theorem}
\label{th:useborder}
Let $B$ be a border for $x_0$.
For all $x_1$ with $x_0 < x_1$,
the following are equivalent:
\begin{enumerate}
\item
$x_1\in\uc(x_0)$ (that is, $x_0\prec x_1$);
\item
there is $y\in B$ such that $x_1 = (x_0 \lor y)$
and, for all $z\in B$, if $(x_0\lor z)\leq(x_0\lor y)$ then
$(x_0\lor z) = (x_0\lor y)$.
\end{enumerate}
\end{theorem}

\begin{proof}
Given $x_0\prec x_1$, we can apply the second condition
in the definition of border for $x_0$: 
$\exists y \in B \, (y\leq x_1)$. 
Using Lemma~\ref{l:easy}, $x_0 < (x_0 \lor y) \leq x_1$,
implying $(x_0 \lor y) = x_1$ since $x_0\prec x_1$.
Additionally, assuming $(x_0\lor z)\leq(x_0\lor y)$ 
for some $z\in B$ leads likewise to
$x_0 < (x_0\lor z) \leq (x_0\lor y) = x_1$ and the
same property applies to obtain
$(x_0\lor z) = (x_0\lor y) = x_1$. 

Conversely, again Lemma~\ref{l:easy} gives 
$x_0 < (x_0 \lor y) = x_1$. By Proposition~\ref{p:covers},
there is $z_0\in\uc(x_0)$ with 
$x_0 \prec z_0 \leq (x_0 \lor y) = x_1$. 
We apply the second condition of borders
to $x_0\prec z_0$ to obtain $z_1 \in B$ with $z_1\leq z_0$,
whence $(x_0 \lor z_1) \leq z_0 \leq (x_0 \lor y) = x_1$,
allowing us to apply the hypothesis of this direction:
$(x_0 \lor z_1) \leq (x_0 \lor y)$ with $z_1 \in B$ implies
$(x_0 \lor z_1) = (x_0 \lor y)$ and, therefore,  
$(x_0 \lor z_1) = z_0 = (x_0 \lor y) = x_1$.
That is, $x_1 = z_0 \in\uc(x_0)$.\qed
\end{proof}

Therefore, given an arbitrary element $x_0$ of the lattice, any candidate for being an element of its upper cover has to be obtainable as a join between $x_0$ and a border element ($x_1 = x_0 \lor y$ for some $y \in B$). Moreover, among these candidates, only those that are minimals represent immediate successors:
they come from those $y$ where $(x_0\lor z)\leq(x_0\lor y)$ 
implies $(x_0\lor z) = (x_0\lor y)$, for all $z \in B$.

\subsection{Advancing Borders}\label{ss:Border}

There is a naturally intuitive operation on borders;
if we have a border $B$ for~$x$, and we use it to
compute the upper cover of $x$, then we do not need
 $B$ as such anymore; to update it, seeing that we 
no longer need to forbid the membership of $x$, it
is natural to consider adding $x$ to the border.
If we had a proper border, and we wished to preserve 
the antichain property, the elements to be removed 
would be exactly the upper cover just computed, as these 
are, as we argue below, the only elements comparable 
to $x$ that could be in a proper border. (All elements 
other than $x$ are mutually incomparable, as the 
border was proper to start with.)

\begin{definition}
Given $x\in\LL$ and a border $B$ for $x$, the
standard step for $B$ and $x$ is $B\cup\{x\}-\uc(x)$.
\end{definition}

 Note that this is {\em not} to say that $\uc(x)\subseteq B$;
 elements of $\uc(x)$ may or may not appear in $B$.
 We will apply the standard step always when $B$
 is a border for $x$, but let us point out that
 the definition would be also valid without this 
 constraint, as it consists of just some set-theoretic
 operations.

\begin{proposition}
Let $B$ be a proper border for $x$.
Then the standard step for $B$ and $x$ 
is also an antichain.
\end{proposition}

\begin{proof}
Elements of the standard step different from $x$
and from all elements of $\uc(x)$ were already
in the previous proper border and are, therefore,
mutually incomparable. None of them is below $x$,
by the first border property. If $y > x$ for some
$y\in B$, then $y\geq z \succ x$ for some $z\in B$,
and the antichain property of $B$ tells us that
$y=z$ so that it gets removed with $\uc(x)$.\qed
\end{proof}

However, we are left with the problem that we have
now a candidate border but we lack the lattice 
element for which it is intended to be a border.
In \cite{MarEkl08} and \cite{BSVG09}, the algorithm
moves on to an intent set of the same cardinality as $x$, 
whenever possible, and to as small as possible 
a larger intent set if all intents of the same 
cardinality are exhausted. In \cite{VaMiLe00}
it is shown that, for their variant of the Border 
algorithm, it suffices to follow a (reversed)
linear embedding of the lattice. Here we follow
this more flexible approach, which is easier now
that we have stated the necessary properties of
borders with no reference to the order of traversal:
there is no need of considering intent sets and their
cardinalities.

Both lattices and their Hasse diagrams can be
seen as directed acyclic graphs, by orienting
the inequalities in either direction; here we
choose to visualize edges $(x,y)$ as corresponding 
to $x\leq y$. A linear embedding corresponds to the 
well-known operation of topological sort of directed 
acyclic graphs, which we will employ for lattices in a 
``reversed'' way:

\begin{definition}
A reverse topological sort of $\LL$
is a total ordering $x_1,\ldots,x_n$ of $\LL$
such that $x_i\leq x_j$ always implies $j\leq i$.
\end{definition}

All our development could be performed with a standard 
topological sort, not reversed, that is, a linear 
embedding of the lattice's partial order. However, 
%a basic application will be for formal concept lattices, where we will focus on comparing intents; 
as it is customary in FCA to guide the visualization through 
the comparison of extents, the algorithms we build 
on were developed with a sort of ``built-in reversal'' 
that we inherit through reversing the topological sort
(see the similar discussion in Section~2.1 of~\cite{BSVG09}).
A reversed topological sort must start with $\top$, 
 hence the initialization is easy:

\begin{proposition}
\label{p:initial}
$B=\emptyset$ 
is a border for $\top\in\LL$.
\end{proposition}

 \begin{proof}
 Both conditions in the definition of border
 become vacuously true: the first one as $B=\emptyset$
 and the second one as the top element has no upper covers.\qed
 \end{proof}

\begin{theorem}
\label{th:advancing}
Let $x_1,\ldots,x_n$ be
a reverse topological sort of $\LL$.
Starting
with $B_1 = \emptyset$, define inductively
$B_{k+1}$ as the standard step for $B_k$ and $x_k$.
Then, for each $k$, $B_k$ is a border for $x_k$.
\end{theorem}

For clarity, we factor off the proof of the following
inductive technical fact, where we use the same notation 
as in the previous statement.

\begin{lemma}
$B_k\subseteq\{ x_1,\ldots,x_{k-1} \}$ and, 
for all 
 $x_j$ with 
$j<k$, there is $y\in B_k$ with $y\leq x_j$.
\end{lemma}

\begin{proof}
For $k=1$, the statements are vacuously true.
Assume it true for $k$, and consider
$B_{k+1} = B_k\cup\{x_k\}-\uc(x_k)$, 
the standard step for $B_k$ and $x_k$.
The first statement is clearly true.
For the second, $x_k$ is itself in
$B_{k+1}$ and, for the rest, inductively,
there is $y\in B_k$ with $y\leq x_j$.
We consider two cases;
if $y\notin\uc(x_k)$, then the same $y$
remains in $B_{k+1}$; otherwise, $x_k\prec y\leq x_j$,
and $x_k$ is the corresponding new $y$ in $B_{k+1}$.\qed
\end{proof}

% \smallskip

\begin{proof}[of Theorem~\ref{th:advancing}]
Again by induction on $k$; we see that the basis is
Proposition~\ref{p:initial}. 
Assuming that $B_k$ is a border for $x_k$, we consider
$B_{k+1} = B_k\cup\{x_k\}-\uc(x_k)$. 
Applying the lemma, 
$B_{k+1}\subseteq\{ x_1,\ldots,x_k \}$,
which ensures immediately that 
$\forall y\in B_{k+1} \, (y \not\leq x_{k+1})$
by the property of the reverse topological sort,
and the first condition of borders follows.
For the second, pick any $z\in\uc(x_{k+1})$; 
by the condition of reverse topological sort, 
$z$, being a strictly larger element than~$x_{k+1}$, 
must appear earlier than it, so that $z=x_j$ with $j<k+1$.
Then, again the lemma tells us immediately
that there is $y\in B_{k+1}$ with $y\leq x_j=z$, 
as we need to complete the proof.\qed
\end{proof}

\subsection{The Generalized Border Algorithm}\label{ss:GBA}

The algorithm we end up validating through our theorems
has almost the same high-level description as the
rendering in \cite{BSVG09}; the most conspicuous
differences are: first, that a reverse topological
sort is used to initialize the traversal of
the lattice; and, second, that
the ``reversed lattice'' model in \cite{BSVG09}
has the consequence that their set-theoretic
intersection in computing candidates becomes
a lattice join in our generalization. Another
minor difference is that Proposition~\ref{p:initial}
spares us 
% , clearly, 
the
separate handling of the first element of the
lattice.
% is unnecessary if the border is
% initially empty.

% \medskip

\begin{algorithm}[H]
RevTopSort($\LL$)\;
$B$ = $\emptyset$\;
$H$ = $\emptyset$\;
\For{$x$ \emph{in} $\LL$, \emph{according to the sort}}
{
	candidates = $\{ x \lor y \st y \in B \}$\;
	cover = minimals(candidates)\;
	\lFor{$z$ in {\rm cover}}{add $(x,z)$ to $H$}\;
	$B$ = $B\cup \{x\}-{}$cover\;
}
\medskip
\caption{The Generalized Border Algorithm}\label{a:GBA}
\end{algorithm}

\medskip

Theorem~\ref{th:advancing} and Proposition~\ref{p:initial}
tell us that the following invariant is maintained:
$B$ is a border for $x$. Then, 
the Hasse edges are computed and added to $H$
according to 
Theorem~\ref{th:useborder}, in two steps:
first, we prepare the list of joins $x \lor y$
and, then, we keep only the minimal elements 
in it. In essence, this process is the same 
as described 
(in somewhat different renderings)
in \cite{BSVG09}, \cite{MarEkl08} or \cite{VaMiLe00}; 
however,
while the definition of border given in \cite{VaMiLe00} 
(and recalled in \cite{BSVG09}) 
leads, eventually, to the same notion
employed in this paper, further development of a general algorithm
that works outside the formal concept analysis framework is dropped
off from \cite{VaMiLe00} on efficiency considerations. 
Moreover, the border algorithm described in
\cite{MarEkl08} works exclusively on the set 
of intents and assumes
the elements are sorted sizewise.
The validations of the algorithms 
in these references rely
very much, at some points, on the fact 
that the lattice is a
sublattice of a powerset and contains formal
concepts, explicitly operating set-theoretically
on their intents.
Theorem~\ref{th:useborder} captures the essence of
the notion of border and lifts 
the algorithm to arbitrary lattices.

One additional difference comes from the fact
that the cost of computing the meet and join 
operations plays a role in the complexity
analysis, but is not available in the general case.
If we assume that meet and join
operations take constant time, then
the total running time of the algorithm 
(except for the sort initialization, 
which takes $\mathcal{O}(|\LL|\log|\LL|)$)
is bounded by
$\mathcal{O}(|\LL|w(\LL)^2)$.
By comparison with \cite{VaMiLe00},
one can see that one factor
of the formula given in \cite{VaMiLe00} gets
dropped under the constant time assumption 
for computing meet and join.
However, this assumption may be unreasonable
in certain applications; the same reference
indicates
that their FCA target case requires a considerable 
amount of graph search for the same operations.
Nevertheless, in absence of further information about
the specific lattice at hand, it is not possible
to provide a finer analysis.

 We must point out that, in our implementation, 
 we have employed a heapsort-based version
 that keeps providing us the next element to handle
 by means of an iterator, instead of completing 
 the sorting step for the initialization.
% Further details on this implementation will be described elsewhere.

\vfill % otherwise the page break may lead to ugly spacing

\section{Distributivity and the iPred Algorithm}\label{s:iPA}

In \cite{BSVG09}, an extra sophistication is introduced
that, as demonstrated both formally in the complexity
analysis of the algorithm and also practically, leads
to a faster algorithm; namely, if some further information 
is maintained along, once the candidates are available 
there is a constant-time test to pick those that are
in the cover, by employing the duality 
$y\in\uc(x)\Leftrightarrow x\in\lc(y)\Leftrightarrow x\prec y$.
Constant time also suffices to maintain 
the additional information. This gives the iPred algorithm.
However, it seems that the unavoidable price is to work on 
formal concepts, as the extra information is heavily
set-theoretic (namely, a union of set differences of 
previously found cover sets for the candidate under study).

Again we show that a fully abstract, lattice-theoretic
interpretation exists, and we show that the essential
property that allows for the algorithm to work is
distributivity: be it due to a distributive $\LL$, or,
as in fact happens in iPred, due to the embedding of
the lattice into a distributive lattice, in the same way 
as concept lattices (possibly nondistributive)
can be embedded in the distributive powerset lattice.

We start treating the simplest case, of very limited
usefulness in itself but good as stepping stone towards
the next theorem. The property where distributivity can
be applied later, if available, is as follows:

\begin{proposition}
\label{p:nondist}
Consider two comparable elements, $x < z$, from $\LL$; 
let $Y\subseteq\lc(z)$ be the set of lower covers
of $z$ that show up in the reverse topological sort 
before~$x$ (it could be empty).
Then, $x\in\lc(z)$ if and only if
$\bigwedge_{y\in Y} (x \lor y) \geq z$.
\end{proposition}

\begin{proof}
Applying Proposition~\ref{p:covers},
we know that there is some $y\in\lc(z)$
such that $x\leq y\prec z$.
Any such $y$, if different from $x$,
must appear before $x$ in
the reverse topological sort.

Suppose first that no lower covers of $z$ appear
before $x$, that is, $Y=\emptyset$. Then, no such 
$y$ different from $x$ can exist; we have that
both $x = y \prec z$ and $\bigwedge_{y\in Y} (x \lor y) = \top \geq z$ 
trivially hold.

In case $Y$ is nonempty, assume first $x \prec z$;
we can apply Lemma~\ref{l:twoprecs}:
$x \lor y=z$ for every $y\in Y$, hence
$\bigwedge_{y\in Y} (x \lor y) = z$. 
To argue the converse, assume $x\notin\lc(z)$ and let
$x\leq y'\prec z$ as before, where we know further
that $x\neq y'$: then $y'\in Y$, so that
$\bigwedge_{y\in Y} (x \lor y) \leq (x \lor y') = y' < z$.\qed
\end{proof}

This means that the test for minimality 
of Algorithm \ref{a:GBA} can be replaced
by checking the indicated inequality; but it is unclear
that we really save time, as a number of joins have to
be performed (between the current element $x$ and all 
the elements in the lower cover of the candidate $z$ 
that appeared before $x$ in the reverse topological sort) 
and the meet of their results computed.
However, clearly, in distributive lattices the test 
can be rephrased in the following, more convenient form:

\begin{proposition}
Assume $\LL$ distributive.
In the same conditions as in the previous proposition,
$x$ is in the lower cover of $z$ if and only if 
$x \lor (\bigwedge_{y\in Y} y) \geq z$.
\end{proposition}

This last version of the test is algorithmically useful: as we
keep identifying elements $Y = \{y_1,\ldots,y_m\}$ of $\lc(z)$, we can
maintain the value of $y = \bigwedge_{i\in \{1,\ldots,m\}} y_i$; then, 
we can test a candidate $z$                      % in constant time 
by computing $x\lor y$ and comparing this value to $z$. 
Afterwards, we update $y$ to $y\land x$
if $x = y_{m+1}$ is indeed in the cover. This may save the loop
that tests for minimality at a                  % very 
small price.

However, unfortunately, if the lattice is not
distributive, this faster test may fail:
given $Y \subseteq \lc(z)$,
the cover elements found so far along the
reverse topological sort, it is always true that 
$x$ is in the lower cover of $z$ if 
$x \lor (\bigwedge_{y \in Y} y) \geq z$,
because 
$z \leq x \lor (\bigwedge_{y \in Y} y) \leq \bigwedge_{y\in Y} (x \lor y)$
and, then, one of the directions of Proposition~\ref{p:nondist} applies;
but the converse does not hold in general. 
Again an example is furnished by 
Figure~\ref{fg:lattices}(a), 
one of the basic,
standard examples of a small nondistributive lattice; 
assume that the 
traversal follows the natural ordering of the
labels, and consider what happens after seeing
that 1 and 2 are indeed lower covers of $z=\top$.
Upon considering $x=3$, we have $Y = \{1,2\}$, so that 
$x \lor (\bigwedge Y) = x \lor \bot = x < z$,
yet $x$ is a lower cover of $z$ and, in fact,
$\bigwedge_{y\in Y} (x \lor y) = (3\lor 1)\land(3\lor 2) = \top$.
Hence, the 
distributivity condition is necessary for the
correctness of the faster test.

\subsection{The Generalized iPred Algorithm}\label{ss:GiPA}

The aim of this subsection is to show the
main contribution of this paper:
we can spare the loop that tests candidates for
minimality in an indirect way,
% moved to the conclusions:
% The iPred algorithm
% uses set-theoretic operations and, therefore, is
% operating with sets that do not belong to the 
% closure space: effectively, it has moved out of
% the concept lattice into the (distributive)
% powerset lattice. This extension can be done
% in general, 
whenever a distributive lattice is
available where we can embed~$\LL$.
However, we must be careful in how the embedding
is performed: the right tool is an injective homomorphism 
of join-semilattices. Recall that, often, this will \emph{not} be
a lattice morphism. Such an example is the identity morphism having as domain the carrier set of a concept lattice $\LL$ over the set of attributes $X$, and as range, $\mathcal{P}(X)$ (see Section \ref{s:Conc} for more details on this particular case).

\begin{theorem}\label{thm:iPred}
Let $(\LL',\leq,\lor)$ be a distributive join-semilattice and 
$f : \LL \rightarrow \LL'$ an injective %$(\leq,\lor)$-
homomorphism.
Consider two comparable elements, $x < z$, from $\LL$; 
let $Y\subseteq\lc(z)$ be the set of lower covers
of $z$ that show up in the reverse topological sort 
before~$x$.% (it could be empty).
Then, $x\prec z$ if and only if
$f(x) \lor (\bigwedge_{y\in Y} f(y)) \geq f(z)$.
\end{theorem}

\begin{proof}
If $Y = \emptyset$ we have $x \prec z$ as in
Proposition~\ref{p:nondist}; for this case,
$\bigwedge_{y\in Y} f(y) = \top$ (of $\LL'$) 
and $f(x) \lor (\bigwedge_{y\in Y} f(y))  
= f(x) \lor \top = \top \geq f(z)$.

For the case where $Y\neq\emptyset$, 
assume first $x\prec z$ and apply Proposition~\ref{p:nondist}:
we have that 
$\bigwedge_{y\in Y} (x \lor y) \geq z$ whence
$f(\bigwedge_{y\in Y} (x \lor y)) \geq f(z)$. 
By Lemma~\ref{l:joinmorph}, we obtain
$f(z) \leq f(\bigwedge_{y\in Y} (x \lor y)) \leq 
\bigwedge_{y\in Y} f(x \lor y) = 
\bigwedge_{y\in Y} (f(x) \lor f(y)) = 
f(x) \lor \bigwedge_{y\in Y} f(y)$,
where we have applied that $f$ commutes with join
and that $\LL'$ is distributive.

For the converse, arguing along the same lines
as in Proposition~\ref{p:nondist},
assume $x\notin\lc(z)$ and let
$x\leq y'\prec z$ with $x\neq y'$ so that $y'\in Y$:
necessarily $\bigwedge_{y\in Y} f(y) \leq f(y')$,
so that 
$f(x) \lor (\bigwedge_{y\in Y} f(y)) \leq f(x) \lor f(y') 
= f(x\lor y') = f(y') < f(z)$, where the last
step makes use of injectiveness.\qed
\end{proof}

The generalized iPred algorithm is based on this
theorem, which proves it correct. 
In it, the 
% join-semilattice
homomorphism $f$ is assumed available, and table LC keeps,
for each $z$, the 
% join -- is it the join? I think it is the meet 
meet of the $f(x)$'s for all the lower
covers $x$ of $z$ seen so far.

\medskip

\begin{algorithm}[H]
RevTopSort($\LL$)\;
$B$ = $\emptyset$\;
$H$ = $\emptyset$\;
\For{$x$ \emph{in} $\LL$, \emph{according to the sort}}
{
	$\LC[x]=\top$\;
	candidates = $\{ x \lor y \st y \in B \}$\;
	\For{$z$ \emph{in candidates}}
		{\If {$f(x) \lor \LC[z] \geq f(z)$}
			{
				add $(x,z)$ to $H$\;
				$\LC[z]=\LC[z] \wedge f(x)$\;
				$B$ = $B - \{z\}$\;
			}}
	$B$ = $B\cup \{x\}$\;
}
\medskip
\caption{The Generalized iPred Algorithm}\label{a:GiPred}
\end{algorithm}

\medskip
 
In the Appendix below, we provide some example runs 
for further clarification.
Regarding the time complexity, again 
we lack information about the cost of meets, joins, and 
comparisons in both
lattices, and also about the cost of computing the homomorphism. 
Assuming constant time for these operations, the 
running time of the generalized iPred algorithm is 
$\mathcal{O}(|\LL|w(\LL))$ (plus sorting): the main loop (line 4-15) is repeated $|\LL|$ times, and then for each of the at most $w(\LL)$ candidates, the algorithm checks if a certain condition is met (in constant time) and updates the diagram and the border in the positive case. 

If meets and joins do not take
constant time, there is little to say
at this level of generality; however, 
for the particular case of  the original
iPred, which only works for lattices of formal concepts,
see \cite{BSVG09}: in the running time analysis there,
one extra factor appears since the meet operation
(corresponding to a set union plus a closure 
operation) is not guaranteed to work in constant time.

\section{Conclusions and Future Work}\label{s:Conc}

We have provided a formal framework for the task of
computing Hasse diagrams of arbitrary lattices through the notion of
``border associated with a lattice element''. Although the concept of
\emph{border} itself is not new, our approach provides a different,
more ``axiomatic'' point of view that facilitates considerably
the application of this notion to algorithms that
construct Hasse diagrams outside the formal concept analysis world.

While Algorithm \ref{a:GBA} is a clear, straightforward generalization
of the Border algorithm of \cite{VaMiLe00,BSVG09} (although the
correctness proof is far less straightforward),
% (generalization
% occurs at the level 
% of generating the 
% candidate set and computing
% minimals among them), 
we consider that we should 
explain further in what sense the iPred
algorithm comes out as a particular case of Algorithm \ref{a:GiPred}.
In fact,
the iPred algorithm
uses set-theoretic operations and, therefore, is
operating with sets that do not belong to the 
closure space: effectively, it has moved out of
the concept lattice into the (distributive)
powerset lattice. 
% This extension can be done in general, 
% The original iPred algorithm takes as input 
Starting from a concept lattice
$(\LL,\leq, \lor,\wedge)$ on a set $X$ of attributes, 
% the embedding into the powerset lattice is
we can define: % , as usual:

\begin{itemize}
\item $x \leq y \Leftrightarrow x \supseteq y$
\item $x \lor y := x \cap y$
\item $x \land y := \bigvee \{z \in \LL \st z \leq x, z \leq y\} = \bigcap \{z \in \LL \st z \supseteq x, z \supseteq y\}$
\item $\top:=\emptyset, \bot:=X$
\end{itemize}

Thus, $\LL$ is a join-subsemilattice of the (reversed)
powerset on $X$, and we can define $f : \LL \rightarrow \mathcal{P}(X)$ as
the identity function: it is injective, and it is a join-homomorphism
since $\LL$, being a concept lattice, is closed under
set-theoretic intersection. Therefore, Theorem \ref{thm:iPred} can 
be translated to: $x \in \lc(z)$ if and only if 
$x \cap (\bigcup_{y \in Y} y) \subseteq z$, where $Y$ is the set of lower covers of $z$ 
already found; this is fully equivalent to the condition behind 
algorithm iPred of \cite{BSVG09} (see Proposition 1 on page 169 
in \cite{BSVG09}). Additionally, iPred works on one specific
topological sort, where all intents of the same cardinality
appear together; our generalization shows that this is not 
necessary: any linear embedding suffices.

A further application we have in mind refers to various forms of implication known as multivalued 
dependency clauses \cite{SDPF81,SDPF87}; in \cite{Baix07,Baix08,BaiBal06},
these clauses are shown to be related to partition
lattices in a similar way as implications are related
to concept lattices through the Guigues-Duquenne basis
(\cite{GanWil99,GuiDuq86}); further, certain database dependencies
(the degenerate multivalued dependencies of \cite{SDPF81,SDPF87})
are related to these clauses in the same way as
functional dependencies correspond to implications.
Data Mining algorithms that extract multivalued
dependencies do exist \cite{SavFla00} but we believe
that alternative ones can be designed using Hasse
diagrams of the corresponding partition lattices
or related structures like split set lattices \cite{Baix07}.
The task is not immediate, as functional and
degenerate multivalued dependencies are of the
so-called ``equality-generating'' sort but
full-fledged
multivalued dependencies are of the so-called
``tuple-generating'' sort, and their connection
to lattices is more sophisticated (see \cite{Baix07});
but we still hope that further work along this
lattice-theoretic approach to Hasse diagrams would
allow us to create a novel application to 
multivalued dependency mining.

\section{Appendix}

We exemplify here some runs of iPred, for the sake
of clarity. First we see how it operates on the
lattice in Figure~\ref{fg:lattices}(a), 
denoted $\LL$ here, using 
as $f$ the injective homomorphism into 
the distributive lattice of
Figure~\ref{fg:lattices}(b) provided by the labels. 
The run is reported in Table~\ref{tb:run}, where
we can see that we identify the respective upper covers
of each of the lattice elements in turn. The linear
order is assumed to be $( \top, 1, 2, 3, \bot )$.
Only the last loop has more than one candidate,
in fact three. The snapshots of the values of 
$B$, $H$, and $\LC$ reported in each row 
(except the initialization) are taken at the end 
of the corresponding loop, so that each reported
value of $B$ is a border for the next row. 
In the Hasse edges $H$, thin lines represent
edges that are yet to be found, and thick lines
represent the edges found so far.
Recall that the values
of $\LC$ are actually elements of the distributive
lattice of Figure~\ref{fg:lattices}(b), and not from $\LL$.

% mini Hasse diagrams for the first example
\def\1{\begin{tikzpicture}
\draw [white] (0,0.25) -- (0,0.2); 
% to leave some space above (should learn to use boundingbox instead...)
\draw [ultra thin] 
(0,0.2) -- (-0.2,0) (0,0.2) -- (0,0) (0,0.2) -- (0.2,0)
(-0.2,0) -- (0,-0.2) (0,0) -- (0,-0.2) (0.2,0) -- (0,-0.2);
\end{tikzpicture}}
\def\2{\begin{tikzpicture}
\draw [white] (0,0.25) -- (0,0.2); 
\draw [very thick] (0,0.2) -- (-0.2,0); 
\draw [ultra thin] (0,0.2) -- (0,0) (0,0.2) -- (0.2,0)
(-0.2,0) -- (0,-0.2) (0,0) -- (0,-0.2) (0.2,0) -- (0,-0.2);
\end{tikzpicture}}
\def\3{\begin{tikzpicture}
\draw [white] (0,0.25) -- (0,0.2); 
\draw [very thick] (0,0.2) -- (-0.2,0) (0,0.2) -- (0,0);  
\draw [ultra thin] (0,0.2) -- (0.2,0) -- (0,-0.2) -- (-0.2,0) 
(0,0) -- (0,-0.2);
\end{tikzpicture}}
\def\4{\begin{tikzpicture}
\draw [white] (0,0.25) -- (0,0.2); 
\draw [very thick] (-0.2,0) -- (0,0.2) (0,0.2) -- (0,0) (0,0.2)-- (0.2,0);
\draw [ultra thin] (-0.2,0) -- (0,-0.2) -- (0,0) (0.2,0) -- (0,-0.2);
\end{tikzpicture}}
\def\5{\begin{tikzpicture}
\draw [white] (0,0.25) -- (0,0.2); 
\draw [very thick] 
(0,0.2) -- (-0.2,0) (0,0.2) -- (0,0) (0,0.2) -- (0.2,0)
(-0.2,0) -- (0,-0.2) (0,0) -- (0,-0.2) (0.2,0) -- (0,-0.2);
\end{tikzpicture}}

\begin{table}
\begin{center}
\begin{tabular}{|c|c|c|c|c|c|c|c|c|}
\hline
$\LL$  & $B$ & $H$ & cand & $\LC[\top]$&$\LC[1]$&$\LC[2]$&$\LC[3]$&$\LC[\bot]$\\
\hline
init   & $\0$ &\1&      &            &        &        &        &           \\
\hline
$\top$ &$\{\top\}$&\1& $\0$ & $\top$   &        &        &        &           \\
1      &$\{1\}$&\2& $\{\top\}$ &   1  & $\top$ &        &        &           \\
2      &$\{1,2\}$&\3& $\{\top\}$ & 4  & $\top$ & $\top$ &        &           \\
3      &$\{1,2,3\}$&\4& $\{\top\}$ & $\bot$ & $\top$ & $\top$ & $\top$ &     \\
$\bot$ & $\0$      &\5& $\{1,2,3\}$ & $\bot$ & $\bot$ & $\bot$ & $\bot$ & $\top$ \\
\hline
\end{tabular}
\medskip
\caption{Example run of the iPred algorithm using the lattices in Figure~\ref{fg:lattices}}
\label{tb:run}
\end{center}
\end{table}

All along the run we can
see that $\LC[z]$ indeed maintains the meet of the set of
predecessors found so far for $f(z)$ in the distributive
embedding lattice; of course, this meet is $\top$ whenever
the set is empty.
 
Let us compare with the run on
the distributive lattice in Figure~\ref{fg:distlattice}, 
where the homomorphism $f$ is now the identity. Observe that
the only different Hasse edge is the one above 3 which
now goes to 2 instead of going to $\top$. Again the 
linear sort follows the order of the labels. 

\begin{figure}
\begin{center}
\begin{tikzpicture}
\draw (0,0) -- (-1,1) -- (0,2) -- (1,1) -- (0,0) -- (1,-1) -- (2,0) -- (1,1);
\draw (0,2) node [fill=white] {$\top$}
      (-1,1) node [fill=white] {1}
      (1,1) node [fill=white] {2} 
      (2,0) node [fill=white] {3} 
      (0,0) node [fill=white] {4} 
      (1,-1) node [fill=white] {$\bot$}; 
\end{tikzpicture}
\caption{A distributive lattice}
\label{fg:distlattice}
\end{center}
\end{figure}
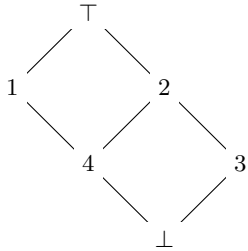

Due to the similarity among the Hasse diagrams, the
run of generalized iPred on this lattice starts 
exactly like the one already given, 
% for the nondistributive
% lattice in Figure~\ref{fg:lattices}(a), 
up to the point 
where node 3 is being processed. At that point, 2 is
candidate and will indeed create an edge, but 1 leads
to candidate $1\lor3=\top$ for which the test fails,
as $\LC[\top] = 4$ at that point, and $3\lor 4 = 2 < \top$.
Hence, this candidate has no effect. After this, the visits
to 4 and $\bot$ complete the Hasse diagram with their
corresponding upper covers.

% mini Hasse diagrams for the second example
\def\1{\begin{tikzpicture}
\draw [white] (0,0.35) -- (0,0.3); 
\draw [ultra thin] 
(0.15,0.15) -- (0,0.3) -- (-0.15,0.15) -- (0,0) -- (0.15,0.15) -- (0.3,0)
-- (0.15,-0.15) -- (0,0);
\end{tikzpicture}}
\def\2{\begin{tikzpicture}
\draw [white] (0,0.35) -- (0,0.3); 
\draw [very thick] (0,0.3) -- (-0.15,0.15); 
\draw [ultra thin] 
(0.15,0.15) -- (0,0.3) (-0.15,0.15) -- (0,0) -- (0.15,0.15) -- (0.3,0)
-- (0.15,-0.15) -- (0,0);
\end{tikzpicture}}
\def\3{\begin{tikzpicture}
\draw [white] (0,0.35) -- (0,0.3); 
\draw [very thick] 
(0,0.3) -- (-0.15,0.15) (0,0.3) -- (0.15,0.15); 
\draw [ultra thin] 
(-0.15,0.15) -- (0,0) -- (0.15,0.15) -- (0.3,0)
-- (0.15,-0.15) -- (0,0);
\end{tikzpicture}}
\def\4{\begin{tikzpicture}
\draw [white] (0,0.35) -- (0,0.3); 
\draw [very thick] 
(0,0.3) -- (-0.15,0.15) (0,0.3) -- (0.15,0.15) -- (0.3,0);
\draw [ultra thin] 
(-0.15,0.15) -- (0,0) -- (0.15,0.15) 
(0.3,0) -- (0.15,-0.15) -- (0,0);
\end{tikzpicture}}
\def\5{\begin{tikzpicture}
\draw [white] (0,0.35) -- (0,0.3); 
\draw [very thick] 
(0,0.3) -- (-0.15,0.15) (0,0.3) -- (0.15,0.15) -- (0.3,0)
(-0.15,0.15) -- (0,0) (0.15,0.15) -- (0,0);
\draw [ultra thin] 
(0.3,0) -- (0.15,-0.15) -- (0,0);
\end{tikzpicture}}
\def\6{\begin{tikzpicture}
\draw [white] (0,0.35) -- (0,0.3); 
\draw [very thick] 
(0,0.3) -- (0.15,0.15) -- (0.3,0)
(-0.15,0.15) -- (0,0) -- (0.15,-0.15) -- (0,0) 
(0,0.3) -- (-0.15,0.15) (0.15,0.15) -- (0,0) (0.3,0) -- (0.15,-0.15);
\end{tikzpicture}}

\begin{table}
\begin{center}
\begin{tabular}{|c|c|c|c|c|c|c|c|c|c|}
\hline
$\LL$  & $B$ & $H$ & cand & $\LC[\top]$&$\LC[1]$&$\LC[2]$&$\LC[3]$&$\LC[4]$&$\LC[\bot]$\\
\hline
init   & $\0$ &\1&      &            &        &        &        &        &   \\
\hline
$\top$ &$\{\top\}$&\1& $\0$ & $\top$   &        &        &        &     &    \\
1      &$\{1\}$&\2& $\{\top\}$ &   1  & $\top$ &        &        &      &    \\
2      &$\{1,2\}$&\3& $\{\top\}$ & 4  & $\top$ & $\top$ &        &     &     \\
3      &$\{1,3\}$&\4& $\{\top,2\}$ & 4 & $\top$ & 3 & $\top$ & &    \\
4      &$\{3,4\}$&\5& $\{1,2\}$ & 4 & 4 & $\bot$ & $\top$ &  $\top$ &   \\
$\bot$ & $\0$      &\6& $\{3,4\}$ & 4 & $\bot$ & $\bot$ & $\bot$ & $\bot$
& $\top$ \\
\hline
\end{tabular}
\medskip
\caption{Example run of the iPred algorithm on the lattice in Figure~\ref{fg:distlattice}}
\label{tb:secondrun}
\end{center}
\end{table}

\bibliographystyle{splncs03}
\bibliography{bibfile}

\end{document}